\documentclass[twoside,11pt]{article}

\usepackage[abbrvbib,preprint]{jmlr_style/jmlr2e}
\hypersetup{hypertexnames=false}
\usepackage{algorithm}
\usepackage{algorithmic}
\usepackage{amsmath}
\usepackage{amsfonts}
\usepackage[version=4]{mhchem}
\usepackage{stmaryrd}
\usepackage{bbold}
\usepackage[export]{adjustbox}
\usepackage[dvipsnames]{xcolor}
\usepackage{float}
\usepackage{placeins}
\usepackage{booktabs}
\usepackage{multirow}
\usepackage[normalem]{ulem}
\usepackage{enumitem}
\usepackage{lastpage}

\graphicspath{ {./} }

\newcommand{\del}[1]{}
\jmlrheading{0}{2026}{1-\pageref{LastPage}}{6/26}{--}{submitted}{Zhang, Zheng, and Zhang}
\ShortHeadings{Robust-Ensemble Soft Actor-Critic}{Zhang, Zheng, and Zhang}
\firstpageno{1}

\begin{document}

\title{Robust-Ensemble Soft Actor-Critic (RE-SAC): Target Aggregation and Value Stability in Noisy Control}

\author{\name Yifan Zhang \email erzhu419@gmail.com; 204201048@csu.edu.cn\\
\addr Central South University\\
Changsha, Hunan, China
\AND
\name Liang Zheng \email zhengliang@csu.edu.cn\\
\addr Central South University\\
Changsha, Hunan, China
\AND
\name Qifan Zhang \email zhangqifan0419@gmail.com\\
\addr Faculty of Information Technology\\
Monash University Malaysia\\
Malaysia}

\maketitle

\begin{abstract}
Clipped critic targets can suppress value estimates in stochastic control, but raising a critic's value need not improve its policy. We study this distinction in Robust-Ensemble Soft Actor-Critic (RE-SAC), separating target aggregation, actor-gradient routing, and a detached continuation norm bonus. Finite-ensemble identities relate the minimum target to a nonnegative mean-minus-min deduction; fixed-policy bonus identities distinguish task value from the augmented training objective. These statements are machine-verified in Lean. In a matched three-seed, ten-critic noisy-bus control, replacing minimum targets with mean targets improves late cumulative reward by $24{,}369$ and fresh-evaluation reward by $32{,}830$ on average. Native task-Q signed bias changes from $-673.493$ to $+16.295$, with improvement in every seed. A separate action-path gradient control also improves return. However, the corrected full norm-augmented method underperforms its no-norm bundle in every seed. Archived bus comparisons and five-seed continuous-control sweeps provide complementary performance and failure-boundary evidence. The results support a clipping-associated underestimation mechanism in the tested bus protocol, not statistical identification of aleatoric and epistemic uncertainty or an independent norm-channel benefit.
\end{abstract}


\section{Introduction}
Off-policy actor--critic methods reuse experience efficiently, but their value estimates can become excessively pessimistic under noisy dynamics. Clipped double-Q targets reduce overestimation~\citep{Fujimoto2018TD3}, whereas persistent pessimism can inhibit exploitation of successful trajectories~\citep{Ji2024SeizingSerendipity}. In an ensemble, the minimum is the mean minus a nonnegative, prediction-dependent deduction. Understanding that deduction requires separating common prediction error, differential critic error, and changes to the value objective.

Bus holding provides a demanding instance of this problem. Passenger arrivals and travel speeds remain stochastic even in frequently visited states. A delay increases passenger accumulation and dwell time, while the following bus catches up, creating the feedback that drives bunching~\citep{Daganzo2009Headway, Xuan2011Dynamic}. Reinforcement learning (RL) can adapt holding decisions~\citep{Wang2020DynamicHolding}, but return curves alone cannot determine whether a value-estimation intervention corrects underestimation or merely shifts the critic's scale.

Uncertainty decompositions distinguish irreducible environmental noise from limited model knowledge~\citep{Depeweg2018Decomposition, Clements2019EstimatingRisk, Hullermeier2021AleatoricEpistemic}. They motivate the original RE-SAC design, which combines diversified critics with a weight-norm proxy. The implemented positive norm bonus, however, changes the continuation objective. A critic may accurately fit that augmented objective while misestimating the task return. Ensemble disagreement and a global network norm are therefore evaluated as construction-dependent signals rather than identified uncertainty components.

\begin{definition}[Value-underestimation diagnostic]
\label{def:poisoning}
For a specified evaluation reference $Q_{\mathrm{ref}}(s,a)$, the signed error is
$e_Q(s,a)=Q_{\mathrm{pred}}(s,a)-Q_{\mathrm{ref}}(s,a)$.
Negative error indicates underestimation relative to that reference.
Native task returns, entropy-augmented returns, and bonus-augmented returns
define different references and are reported separately. The term
\emph{Q-value poisoning} describes the motivating degradation pattern; the
experiments do not observe the optimal value or environmental variance
pointwise.
\end{definition}

This paper studies target aggregation and objective shifts within the RE-SAC research framework, building on the bus simulator of~\citet{Zhang2026SingleAgentBus}. The controlled mean-target comparator is not relabeled as the original full RE-SAC method. We preserve archived performance results, isolate target and gradient-path changes under matched protocols, and report the full-method negative control alongside the positive clipping result.

Our contributions are:
\begin{itemize}[leftmargin=*,itemsep=2pt]
\item \textbf{Finite-ensemble underestimation analysis.} We express a minimum target as a mean target minus a nonnegative deduction, without assuming Gaussian or independent critic errors. Terminal masking is explicit.
\item \textbf{Objective-shift analysis.} Fixed-policy continuation-bonus identities separate task value from augmented value and show how exposure differences can change action gaps. Lean verifies these identities and the corresponding conditional residual bounds.
\item \textbf{Matched mechanism controls.} Three-seed, $K=10$ bus experiments distinguish target aggregation, actor-gradient routing, and the norm bonus. The target intervention improves return and native task-Q calibration, while the norm-augmented full arm has no independent return benefit.
\item \textbf{Applicability evidence.} Archived bus results and five-seed continuous-control sweeps document performance, seed sensitivity, and adverse cases under their original configurations.
\end{itemize}

\section{Related work}
\paragraph{Value estimation and ensemble pessimism.}
Soft Actor-Critic (SAC) combines off-policy learning with entropy regularization~\citep{Haarnoja2018SAC}. Clipped double-Q targets address overestimation~\citep{Fujimoto2018TD3}, while diversified ensembles~\citep{Osband2016BootstrappedDQN, An2021UncertaintyEnsemble} and Randomized Ensembled Double Q-learning (REDQ)~\citep{Chen2021REDQ} offer alternative aggregation mechanisms. Blended Exploitation and Exploration Actor-Critic (BAC) uses successful replay actions to mitigate under-exploitation~\citep{Ji2024SeizingSerendipity}. Our contribution is a protocol-aligned analysis of clipping and objective shifts, not the introduction of mean aggregation itself.

\paragraph{Uncertainty and robust control.}
\citet{Depeweg2018Decomposition} and \citet{Clements2019EstimatingRisk} motivate separate treatment of environmental noise and model uncertainty. Robust Markov decision processes (RMDPs) optimize over transition ambiguity sets~\citep{Iyengar2005RobustDP, Wiesemann2013RobustMDP}; risk-sensitive and regularized formulations supply related modeling perspectives~\citep{Osogami2012Robustness, Geist2019TheoryRegularized}. Integral probability metric (IPM) formulations~\citep{Zhou2023NaturalActorCritic} and Lipschitz control~\citep{Asadi2018Lipschitz} motivate a sensitivity proxy, but do not identify the implemented norm bonus with a statistical noise estimate or a robust inner minimization.

\paragraph{Transit control.}
Classical holding policies address bus bunching~\citep{Daganzo2009Headway, Xuan2011Dynamic}, and deep reinforcement learning (DRL) supports adaptive holding~\citep{Wang2020DynamicHolding}. Our testbed is the event-driven bus simulator of~\citet{Zhang2026SingleAgentBus}. The main mechanism claim concerns critic aggregation; passenger-service metrics remain a separate application evaluation.

\section{Preliminaries}
\subsection{Notation and soft policy evaluation}
We consider a discounted Markov decision process (MDP).
The state and action spaces are $\mathcal{S}$ and $\mathcal{A}$; $P$ is the
transition kernel, $R$ the expected task reward, and $\gamma\in[0,1)$ the
discount. A replay transition is $(s,a,r,s',d)$, where $d\in\{0,1\}$ marks
termination. The learner reward $\tilde r$ follows the cohort's fixed
transformation; reported environment returns retain the native reward scale.
The replay buffer is $\mathcal{D}$, the ensemble size is $K$, and
$\bar Q$ and $\sigma_{\mathrm{ens}}$ are the mean and standard deviation of
unshifted critic predictions.

For a fixed policy $\pi$ and entropy temperature $\alpha_{\mathrm{ent}}$,
SAC policy evaluation uses the soft value~\citep{Haarnoja2018SAC}
\begin{equation}
V_Q^\pi(s)=\mathbb{E}_{a\sim\pi(\cdot|s)}
 [Q(s,a)-\alpha_{\mathrm{ent}}\log\pi(a|s)].
\end{equation}
With a bounded entropy payoff, its Bellman operator contracts to the
fixed-policy soft value, not directly to the optimal value.
A finite-action LogSumExp operator instead represents exact
entropy-regularized maximization; it is a separate abstraction.

\section{Methodology}
\subsection{Target aggregation and underestimation}
For the frozen target predictions $q_k=\hat Q_{\phi'_k}(s',a')$, define
$m(q)=K^{-1}\sum_k q_k$, $\ell(q)=\min_k q_k$, and
$D_K(q)=m(q)-\ell(q)\ge0$. Thus
\begin{equation}
\ell(q)-Q_{\mathrm{ref}}=
 [m(q)-Q_{\mathrm{ref}}]-D_K(q),\qquad
y^{\mathrm{mean}}-y^{\mathrm{min}}=\gamma(1-d)D_K(q).
\label{eq:clipping_identity}
\end{equation}
The second identity compares matched inputs, reward transforms, entropy
terms, and terminal indicators. For two critics only,
$D_2=|q_1-q_2|/2$; this absolute-difference formula is not used for $K=10$.
\texttt{Underestimation.\allowbreak lean} verifies both cases and the
nonnegative target increment. Increasing common noise shared by identical
heads does not create a clipping deduction.

The identity separates the mean prediction error from the clipping
deduction. Removing the deduction raises the matched-input target but need
not improve calibration: the new mean prediction may already be optimistic.
Training also changes policy and replay occupancy. We therefore test both
return and native task-Q signed error, rather than equating a target uplift
with successful bias repair.

\subsection{Norm proxy and the legacy robust surrogate}
For a nominal transition $p^\circ$ and a 1-Wasserstein ambiguity ball of
radius $\rho_{\mathrm{ipm}}$, a Lipschitz value function satisfies
\begin{equation}
\inf_{\nu:W_1(\nu,p^\circ)\le\rho_{\mathrm{ipm}}}\mathbb{E}_\nu[V_\phi]
 \ge \mathbb{E}_{p^\circ}[V_\phi]
       -\rho_{\mathrm{ipm}}\mathrm{Lip}(V_\phi).
\label{eq:robust_lower_bound}
\end{equation}
Here $\nu$ is an adversarial transition distribution. This sensitivity
bound motivates the original norm proxy. With
1-Lipschitz activations,
\begin{equation}
\Gamma_{\mathrm{ale}}^{\mathrm{Lip}}(V_\phi)
 =\rho_{\mathrm{ipm}}\mathrm{Lip}(V_\phi)
 \le\rho_{\mathrm{ipm}}\prod_l\|W_l\|_2.
\label{eq:gamma_ale}
\end{equation}
The implemented entrywise $\ell_1$ sum is not asserted to bound this product.
Here $\Gamma_{\mathrm{ale}}^{\mathrm{Lip}}$ names the sensitivity term,
not an estimator of environmental variance.

The legacy Robust-Ensemble Value (REV) surrogate is
\begin{equation}
\mathcal{T}^{REV}_\kappa Q(s,a)
 =R(s,a)+\gamma\left(\mathbb{E}_{P}[V^Q(s')]
             -\lambda_{\mathrm{epi}}\Gamma_{\mathrm{epi}}(s,a)-\kappa\right),
\label{eq:rev_operator}
\end{equation}
where $V^Q$ is a hard-max or finite soft value and the supplied scalar
$\kappa$ and mapping $\Gamma_{\mathrm{epi}}$ are fixed independently of $Q$.
The ensemble variance
\begin{equation}
\Gamma_{\mathrm{epi}}(s,a)
 =\frac{1}{K-1}\sum_k
   \left(\hat Q_{\phi'_k}(s,a)-\bar{\hat Q}(s,a)\right)^2
\label{eq:gamma_epi}
\end{equation}
motivates disagreement control, but neither the legacy theorem nor the
experiments identifies it as purely epistemic uncertainty.

\subsection{Implemented bus targets and actor scores}
The full bus recipe uses a positive continuation bonus, not the negative
penalties in Eq.~\eqref{eq:rev_operator}. For a target head with $L$ affine
layers, weight matrices $W_l^{(\phi'_k)}$, and bias vectors
$v_l^{(\phi'_k)}$, define
\[
N_k=\sum_{l=1}^{L}\|W_l^{(\phi'_k)}\|_{1,\mathrm{entry}}
 +\sum_{l=1}^{L-1}\|v_l^{(\phi'_k)}\|_1,\qquad
b_k^{\mathrm{norm}}=\lambda_{\mathrm{ale}}N_k.
\]
The entrywise norm includes hidden biases but excludes the output bias
and categorical embeddings. The detached bonus computed before the
critic step is reused for the actor step. The target is
\begin{equation}
y_k^{\mathrm{bus}}
 =\tilde r+\gamma(1-d)
   \left[\hat Q_{\phi'_k}(s',a')
       -\alpha_{\mathrm{ent}}\log\pi_\theta(a'|s')
       +b_k^{\mathrm{norm}}\right].
\label{eq:bus_target_y}
\end{equation}
The shared-mean and shared-min controls set $b_k^{\mathrm{norm}}=0$ and
replace each target-head value by $m(q)$ or $\ell(q)$.
The no-norm bundle retains the full headwise target with zero bonus.

For unshifted predictions, the critic disagreement penalty is
$\beta_{\mathrm{ood}}\sigma_{\mathrm{ens}}$. The full actor instead scores
$z_k(s,a)=Q_{\phi_k}(s,a)+b_k^{\mathrm{norm}}$, with mean $\bar z$ and
sample standard deviation $\sigma_z$:
\begin{equation}
\begin{split}
\mathcal{L}_Q&=\mathbb{E}_{\mathcal D}
 [K^{-1}\textstyle\sum_k(Q_{\phi_k}-y_k)^2
       +\beta_{\mathrm{ood}}\sigma_{\mathrm{ens}}],\\
\mathcal{L}_\pi&=\mathbb{E}_{s,a\sim\mathcal D,\ a_\pi\sim\pi_\theta}
 [\alpha_{\mathrm{ent}}\log\pi_\theta(a_\pi|s)
       -(\bar z(s,a_\pi)+\beta_{\mathrm{lcb}}\sigma_z(s,a_\pi))
       +\beta_{\mathrm{bc}}\|a_\pi-a\|^2].
\end{split}
\label{eq:actor_loss}
\end{equation}
Here $\beta_{\mathrm{lcb}}=-2$ and $\beta_{\mathrm{bc}}=0.001$ for the
controlled full/bundle recipe. The Lower Confidence Bound (LCB) score
subtracts disagreement because the loss is minimized. The shared-mean
and shared-min target controls instead use the same minimum-critic actor
and no norm, behavior-cloning (BC), or disagreement penalties.

Stop and vehicle categorical embeddings follow~\citet{Zhang2026SingleAgentBus}.
The action-path gradient control detaches the direct critic-state embedding
path during actor updates while retaining the policy-action derivative
through the shared representation. It changes the gradient route, not the
forward target or the evaluation reference.

\begin{algorithm}[!htbp]
\caption{Bus training with the cohort's frozen target and actor choices}
\label{alg:resac}
\begin{algorithmic}
\STATE Initialize critics, target critics, actor, and replay buffer $\mathcal D$.
\STATE Fix the cohort's aggregation, bonus, actor score, and gradient route.
\FOR{each training update}
 \STATE Sample $(s,a,\tilde r,s',d)$ and next action $a'\sim\pi_\theta$.
 \STATE Compute frozen target predictions $q_k$ and the prescribed bonuses.
 \STATE Choose headwise, shared-mean, or shared-min continuation values.
 \STATE Form targets with entropy and terminal mask as in Eq.~\eqref{eq:bus_target_y}.
 \STATE Update critics with the prescribed disagreement coefficient.
 \STATE Update the actor with the prescribed score, BC coefficient, and gradient route.
 \STATE Apply the cohort's fixed or tuned entropy temperature.
 \STATE Smooth target parameters using $\tau_{\mathrm{tgt}}$.
\ENDFOR
\end{algorithmic}
\end{algorithm}

\subsection{Fixed-policy objective shift}
For a frozen policy, let $P_\pi$ be its substochastic continuation kernel
and $g_\pi$ its bounded task payoff, including the prescribed entropy term.
A frozen scalar bonus $b$ gives the backup
$\mathcal B_b^\pi q=g_\pi+\gamma P_\pi(q+b\mathbf 1)$.
The continuation exposure $H^\pi$ satisfies
$H^\pi=\gamma P_\pi(\mathbf 1+H^\pi)$.
Consequently,
\begin{equation}
Q^{\pi,b}=Q^\pi+bH^\pi,\qquad
0\le H^\pi\le\frac{\gamma}{1-\gamma}.
\label{eq:bonus_decomposition}
\end{equation}
For two actions, the value-gap change is
$b[H^\pi(s,a_1)-H^\pi(s,a_2)]$; a constant per-head bonus can therefore
change action preferences through differing continuation exposure.
\texttt{ContinuationBonus.\allowbreak lean} verifies the finite-horizon
decomposition, exposure bound, fixed-policy backup difference, and
conditional residual bound. The infinite-horizon identity follows under the
bounded fixed-policy assumptions, not from neural-training convergence.

\section{Experiments}
We separate three evidence cohorts: matched noisy-bus mechanism controls,
the archived bus application comparison, and the Brax continuous-control
family. Each retains its own training seeds, target/actor recipe, and
evaluation reference. Return contrasts are paired within a cohort; they
are not pooled across these different algorithms.

\subsection{Controlled noisy-bus tests}
\label{sec:controlled_bus}
The primary controls use the simulator in Section~\ref{sec:bus_experiments},
three seeds: 8400, 8401, and 8402, $K=10$, fixed entropy temperature
$\alpha_{\mathrm{ent}}=0.6$, and 500 training days.
Each fit uses 512,000 critic and 256,000 actor updates. The registered late
window averages days 425, 450, 475, and 500 and two evaluation scenarios within
seed. Fresh calibration uses eight days per model, 256 queries per day,
and 512 reward-normalization co-batches. Reported uncertainties below are
seed standard errors (SE), not variation over queries or rollout days.
The learner uses $\tilde r=10(r-\mu_r)/s_r$, where $\mu_r$ and $s_r$
are the replay minibatch reward mean and sample standard deviation,
including terminal zero rewards; the denominator is replaced by one
when $s_r<10^{-6}$. The native environment return is not normalized.

The clipping contrast changes only shared-mean versus shared-min targets.
It holds the minimum-critic actor, zero penalties, action-path gradient,
initialization, and random streams fixed. The gradient contrast changes
only the direct critic-state encoder path for the shared-mean/min-actor
recipe. The full/bundle contrast holds headwise targets, the LCB actor,
BC coefficient 0.001, critic disagreement coefficient 0.01, and repaired
gradient route fixed; the full arm additionally uses the norm bonus with
coefficient 0.01. These contrasts test different interventions.

\paragraph{Controlled value-estimation evidence.}
The post-review controls provide a more specific test of the Q-underestimation
hypothesis than the aggregate learning curves. Each row below is a separate
three-seed bus contrast with the same evaluation windows; the rows are not
pooled into a single ablation matrix. Replacing the minimum target by a mean
target improves both return and native task-Q calibration in every controlled
seed. The repaired action-path-gradient comparison also improves return, but
it is a gradient-path control rather than a second uncertainty estimate. In
contrast, the completed full-versus-bundle study finds no independent return
benefit for the norm-augmented arm: the corrected difference is negative for
all three seeds. The positive target-clipping result therefore supports a
clipping-associated underestimation mechanism, while the full-arm result sets
the boundary for claims about the norm proxy.

\begin{table}[!htbp]
\centering
\caption{Frozen controlled evidence for the noisy-bus value-estimation claim. Entries are mean $\pm$ seed SE in native cumulative reward; ``late'' averages days 425/450/475/500 and ``fresh'' averages the prespecified fresh evaluation window. The task-Q column reports the corresponding native task-Q bias comparison where available; negative values indicate underestimation. These are separate controls and are not pooled with the headline seed-extension table.}
\label{tab:controlled_value_evidence}
\small
\begin{tabular}{@{}p{0.22\linewidth}p{0.19\linewidth}p{0.19\linewidth}p{0.31\linewidth}@{}}
\toprule
Contrast & Late return difference $\uparrow$ & Fresh return difference $\uparrow$ & Native task-Q evidence \\
\midrule
Mean target $-$ min target ($K{=}10$) & $+24{,}369\pm896$ & $+32{,}830\pm4{,}177$ & $-673.493\to+16.295$; all 3 seeds improve \\
Action-path gradient control & $+5{,}401\pm2{,}821$ & $+11{,}216\pm2{,}467$ & --- \\
Full $-$ bundle, repaired gradient & $-261{,}545\pm173{,}796$ & --- & --- \\
\bottomrule
\end{tabular}
\end{table}

\subsection{Standard continuous control benchmarks}
\label{sec:mujoco}

\subsubsection{Setup}
\label{sec:mujoco_setup}
We evaluate RE-SAC alongside seven baselines on four MuJoCo continuous control tasks simulated via the Brax physics engine~\citep{Freeman2021Brax} with the \texttt{spring} backend: \textbf{Hopper-v2}, \textbf{HalfCheetah-v2}, \textbf{Walker2d-v2}, and \textbf{Ant-v2}. All algorithms share a common architecture (two-hidden-layer multi-layer perceptron (MLP) with 256 hidden units, $\gamma = 0.99$, $\tau_{\mathrm{tgt}} = 0.005$, learning rate $3 \times 10^{-4}$, batch size 256, replay buffer $10^6$) and are trained for 8M environment steps (2000 iterations $\times$ 4000 steps/iteration, 250 gradient updates per iteration).

\paragraph{Continuous-control B0 family.}
The archived headline label \textbf{RE-SAC (B0)} denotes an
environment-specific launch family, not one fixed hyperparameter vector.
All four launch recipes use independent-target mixing
$\alpha_{\mathrm{ind}}=0.75$, evaluation-actor exponential moving average
(EMA) coefficient $\tau_{\mathrm{ema}}=0.005$, adaptive LCB scheduling,
and zero coefficients for the norm shift and replay disagreement. Their remaining settings are:
\begin{center}
\small
\begin{tabular}{lccc}
\toprule
Environment & $K$ & Final LCB $\beta_{\mathrm{end}}$ & Anchor $\lambda_{\mathrm{anc}}$\\
\midrule
Hopper-v2 / Walker2d-v2 & 10 & $-1$ & 0.01\\
HalfCheetah-v2 & 5 & 0 & 0.001\\
Ant-v2 & 10 & $-2$ & 0.01\\
\bottomrule
\end{tabular}
\end{center}
Here $\lambda_{\mathrm{anc}}$ weights the parameter anchor and
$\beta_{\mathrm{end}}$ is the scheduled final LCB coefficient.
The JAX continuation mixes
$\alpha_{\mathrm{ind}}\hat Q_k+(1-\alpha_{\mathrm{ind}})\min_j\hat Q_j$.
Its actor uses $\bar Q+\beta_{\mathrm{lcb}}\sigma_{\mathrm{JAX}}$, where
$\sigma_{\mathrm{JAX}}$ is the population standard deviation across heads.
This differs from the bus sample standard deviation.
The family is reported under its archived B0 label and is distinct from
the matched bus methods.

\paragraph{Exploratory variants.}
Three earlier variants---\textbf{v1} (shared min targets, no anchor), \textbf{v2} (10$\times$ smaller regularisation), \textbf{v5/v5b/v6b} (per-environment-tuned independent targets)---were used in earlier drafts to isolate individual design choices. They are retained in Table~\ref{tab:oracle_mae} (marked $^\dagger$, seed~8 only) purely as ablation reference points. They are \emph{not} the headline method and are not compared against fixed baselines as if they were a tuned method. The headline five-seed numbers always refer to RE-SAC (B0).

\paragraph{Baselines.}
\begin{itemize}[leftmargin=*,itemsep=1pt]
    \item \textbf{SAC}~\citep{Haarnoja2018SAC}: twin-critic ($K{=}2$), auto-tuned entropy.
    \item \textbf{Twin Delayed Deep Deterministic Policy Gradient (TD3)}~\citep{Fujimoto2018TD3}: twin-critic, deterministic policy, delayed updates.
    \item \textbf{Distributional Soft Actor-Critic (DSAC)}~\citep{Ma2025DSAC}: twin quantile critics using Implicit Quantile Networks (IQN, 8 quantiles).
    \item \textbf{BAC}~\citep{Ji2024SeizingSerendipity}: the Blended Exploitation and Exploration (BEE) Actor-Critic baseline, which couples value pessimism with exploitation of past successful trajectories.
    \item \textbf{REDQ}~\citep{Chen2021REDQ}: an ensemble method with random in-target minimization.
    \item \textbf{Soft Actor-Critic with $N$ critics (SAC-N)}~\citep{An2021UncertaintyEnsemble}: a SAC ensemble-min variant with minimum aggregation, corresponding to the diversified-ensemble baseline in An et al.
    \item \textbf{Truncated Quantile Critics (TQC)}~\citep{Kuznetsov2020TQC}: a distributional ensemble baseline that truncates high quantiles to control overestimation.
\end{itemize}
All baselines use the same five training seeds as RE-SAC (B0): $\{8,16,24,32,40\}$.

\subsubsection{Oracle Q-error analysis}
\label{sec:oracle_q}
One diagnostic claim of RE-SAC is that its ensemble may provide useful value representations in data-sparse regions. To test this descriptively, we evaluate the \textbf{Oracle Q-Error}: after training, we roll out 50 evaluation episodes per algorithm and record, at every step, (i) the Q-value predictions from all critic heads, and (ii) the Monte Carlo (MC) return $Q_{\mathrm{MC}}(s_t,a_t) = \sum_{k=0}^{T_{\mathrm{eval}}-t}\gamma^k r_{t+k}$, where $T_{\mathrm{eval}}$ is the evaluation-episode length. This return serves as a model-free reference. Since algorithms learn Q-values at different absolute scales, we apply per-algorithm z-score alignment (rescaling predictions to match the mean and standard deviation of $Q_{\mathrm{MC}}$). For ensemble methods, we report the \emph{Oracle mean absolute error (MAE)}: for each state, the head whose aligned prediction is closest to $Q_{\mathrm{MC}}$ is selected, and the MAE is computed over these best-head predictions. This measures favorable ensemble representational capacity---how well it \emph{could} estimate the recorded return if paired with an ideal head-selection mechanism; it is not a deployable policy metric or calibrated Q-error.

We bin the Oracle MAE by \textbf{Mahalanobis rareness}---the Mahalanobis distance of each observation from the training distribution's empirical mean and covariance---to examine how estimation errors vary with state rarity.

\begin{table}[!htbp]
\centering
\caption{Oracle MAE (seed~8, single-seed Q-error analysis) and final evaluation returns (5-seed mean$\pm$std over the last 25\% of evaluations, seeds $\{8,16,24,32,40\}$). \textbf{Bold}: best value within the comparable block (Oracle MAE columns for seed-8 Q passes; return columns for 5-seed rows only). RE-SAC v1 (shared min targets, seed~8 only) achieves the lowest Q-estimation error on 3/4 envs; RE-SAC (B0, $\alpha_{\mathrm{ind}}{=}0.75$ corrected from sensitivity analysis, 5 seeds) trades per-head accuracy for ensemble diversity. Variants v2, v5 are retained from earlier drafts as single-seed reference points; the headline five-seed RE-SAC number for this paper is the B0 row.}
\label{tab:oracle_mae}
\small
\setlength{\tabcolsep}{5pt}
\renewcommand{\arraystretch}{1.08}
\textbf{(a) Oracle MAE $\downarrow$ (seed 8)}

\begin{tabular}{@{}lcccc@{}}
\toprule
Algorithm & Hop. & HCheetah & Walk. & Ant \\
\midrule
SAC          & 15.5          & 59.0          & 10.1          & 81.7 \\
DSAC         & 16.6          & 45.9          & 10.6          & 107.7 \\
TD3          & \textbf{6.8}  & 47.6          & 10.7          & 61.4 \\
BAC          & 13.4          & 42.2          & 7.5           & 86.5 \\
REDQ         & --            & --            & --            & -- \\
SAC-N         & --            & --            & --            & -- \\
TQC          & --            & --            & --            & -- \\
\midrule
RE-SAC v1$^\dagger$    & 7.9   & \textbf{30.2} & \textbf{7.0}  & \textbf{48.1} \\
RE-SAC v2$^\dagger$    & 18.9  & 58.4          & 10.6          & 65.0 \\
RE-SAC (B0)  & 13.0          & 82.1          & 10.7          & 74.3 \\
RE-SAC v5$^\dagger$ & 15.4   & 78.4          & 14.5$^*$      & 84.6$^{**}$ \\
\bottomrule
\end{tabular}

\vspace{0.6em}
\textbf{(b) Final Eval Return $\uparrow$ (mean$\pm$std, 5 seeds)}

\begin{tabular}{@{}lcccc@{}}
\toprule
Algorithm & Hop. & HCheetah & Walk. & Ant \\
\midrule
SAC          & 600$\pm$46   & 13795$\pm$4451   & 347$\pm$19    & 12000$\pm$2553 \\
DSAC         & 582$\pm$150  & 14531$\pm$6302   & 419$\pm$119   & 9745$\pm$2694 \\
TD3          & 498$\pm$42   & 16733$\pm$6978   & 387$\pm$69    & 11243$\pm$1724 \\
BAC          & 593$\pm$86   & 11866$\pm$4488   & \textbf{690$\pm$400}   & \textbf{19205$\pm$5154} \\
REDQ         & \textbf{604$\pm$99}  & 19387$\pm$5506   & 360$\pm$50    & 14827$\pm$4648 \\
SAC-N         & 447$\pm$34   & 19133$\pm$7456   & 298$\pm$10    & 12039$\pm$3856 \\
TQC          & 474$\pm$92   & 11263$\pm$3430   & 205$\pm$154   & 538$\pm$1070 \\
\midrule
RE-SAC v1$^\dagger$    & 401   & 10618  & 227   & 17398 \\
RE-SAC v2$^\dagger$    & 392   & 15982  & 276   & 3851 \\
RE-SAC (B0)  & 526$\pm$155   & \textbf{19661$\pm$2593}   & 339$\pm$35    & 16273$\pm$7451 \\
RE-SAC v5$^\dagger$ & 660   & 22711  & 1008$^*$ & 27104$^{**}$ \\
\bottomrule
\end{tabular}

\vspace{0.35em}
\parbox{\textwidth}{\footnotesize $^\dagger$Single seed only (seed~8). $^*$v5b ($\lambda_{\mathrm{anc}}{=}0.01$, $\beta_{\mathrm{end}}{=}{-0.5}$); $^{**}$v6b ($\alpha_{\mathrm{ind}}{=}0.5$). REDQ/SAC-N/TQC are post-review additions; Oracle MAE for these requires a seed-matched Q-data collection pass (future work). RE-SAC v5 single-seed numbers should be interpreted in light of the 5-seed sensitivity sweep in \S\ref{sec:sensitivity}, which shows the underlying independent-target design has high seed variance on Ant.}
\end{table}

\begin{table}[!htbp]
\centering
\caption{Non-oracle Q-error aggregators averaged over the four MuJoCo environments (seed~8 Q-data, MAE $\downarrow$). Oracle selects the best head post hoc and reports the best available aligned head prediction. Mean, LCB, and Min are deployable aggregators, but the errors here still use per-algorithm alignment. For RE-SAC, the LCB form is $\bar{Q}+\beta_{\mathrm{lcb}}\sigma_{\mathrm{JAX}}$.}
\label{tab:nonoracle_mae}
\small
\begin{tabular}{l|cccc}
\toprule
Algorithm & Oracle & Mean-head & LCB aggregator & Min-head \\
\midrule
SAC          & 41.6 & 43.0 & 43.5 & 43.3 \\
DSAC         & 45.2 & 46.4 & 46.8 & 46.6 \\
TD3          & 31.6 & 33.8 & 34.9 & 34.2 \\
BAC          & 37.4 & 38.6 & 39.1 & 38.9 \\
RE-SAC (B0)  & 45.0 & 49.1 & 50.0 & 49.8 \\
RE-SAC v1$^\dagger$ & \textbf{23.3} & \textbf{25.0} & \textbf{25.8} & \textbf{25.7} \\
\bottomrule
\multicolumn{5}{l}{\footnotesize $^\dagger$Single-seed Q-estimation reference point, not the headline policy configuration.}
\end{tabular}
\end{table}

Table~\ref{tab:oracle_mae} separates representation diagnostics from policy
performance. RE-SAC v1 has the lowest seed-8 Oracle MAE on three of four
environments and lower Oracle MAE than BAC on all four.
Table~\ref{tab:nonoracle_mae} shows that the ordering largely survives
mean, LCB, and min aggregation under the same alignment.
With five training seeds, the B0 family has the highest reported mean
return on HalfCheetah, is competitive but not dominant on Ant, and is
weaker than BAC on Walker2d and Ant. The v5/v6b rows report exploratory
single-seed peaks. These different configurations do not establish that
one method jointly improves calibrated Q accuracy and policy return.

The per-bin analysis in Fig.~\ref{fig:mujoco_oracle} shows that RE-SAC v1 maintains low seed-8 Oracle error in high-Mahalanobis-rareness bins. This describes aligned representational capacity, not task-Q calibration.

\textbf{What can be inferred from the v1 Q diagnostic?} The v1 result is consistent with two potentially useful mechanisms: the weight-norm term supplies a smoothness prior, while ensemble disagreement supplies a state-action-dependent control signal. However, this single-seed comparison does not identify either term as a direct aleatoric or epistemic estimator. In particular, reducing the regularization strength in the earlier v2 reference point changes several coupled training effects, so its Oracle-MAE difference cannot establish that both channels independently caused the gain. The controlled bus experiments provide the stronger attribution evidence: target aggregation improves the noisy-bus result, whereas the repaired full-versus-bundle comparison does not show an independent norm-channel return benefit.

\textbf{Independent-target variants.} Head-specific targets can preserve
prediction differences that a shared target suppresses. Their interaction
with LCB scoring, EMA, and anchoring is examined in the sensitivity sweep
below; the seed-8 peaks alone do not identify which component caused the gain.

\begin{figure}[!htbp]
    \centering
    \includegraphics[width=\linewidth]{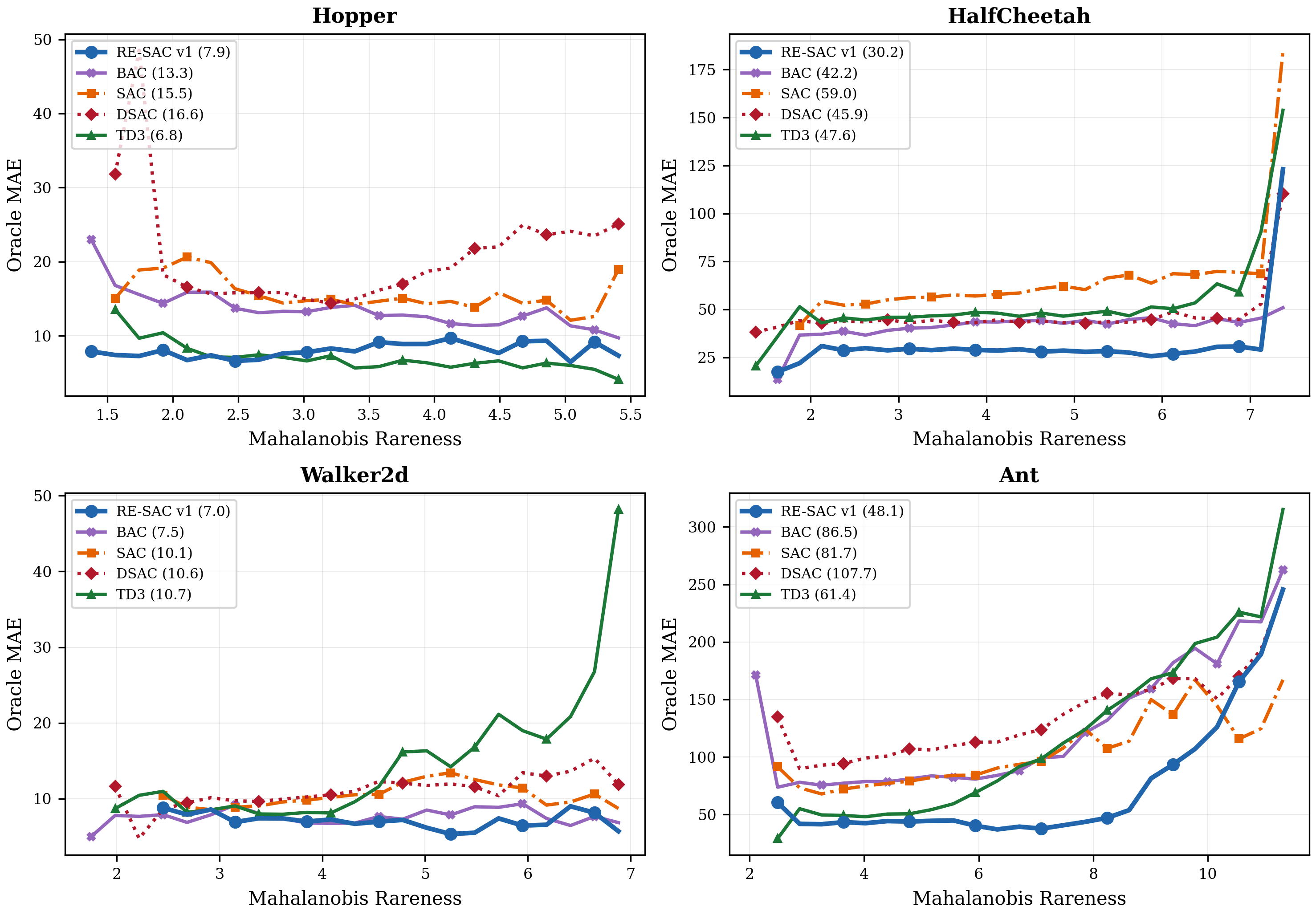}
    \caption{Oracle Q-Error (MAE) binned by Mahalanobis rareness on four MuJoCo environments. Legend entries show overall Oracle MAE. RE-SAC v1 maintains the lowest oracle Q-estimation error across rareness bins on three of four environments, with BAC the closest pessimistic-ensemble competitor on Ant. The figure focuses on the seed-8 Q-error diagnostic algorithms from Table~\ref{tab:oracle_mae}; older exploratory variants (v2/v5/v6b) remain available in the analysis pipeline but are excluded for legend readability.}
    \label{fig:mujoco_oracle}
\end{figure}

\subsubsection{Training performance}
Fig.~\ref{fig:mujoco_curves} shows five-seed evaluation curves for the headline RE-SAC (B0) configuration and the fixed baselines. This visualization matches the statistically comparable block of Table~\ref{tab:oracle_mae}: RE-SAC (B0) is strongest on HalfCheetah, competitive but high-variance on Ant, and weaker than the strongest baseline on Hopper/Walker2d. Single-seed exploratory curves for v1/v2/v5/v5b/v6b are moved to the online appendix as diagnostic evidence about independent-target dynamics, not headline policy-return comparisons.

\begin{figure}[!htbp]
    \centering
    \includegraphics[width=\linewidth]{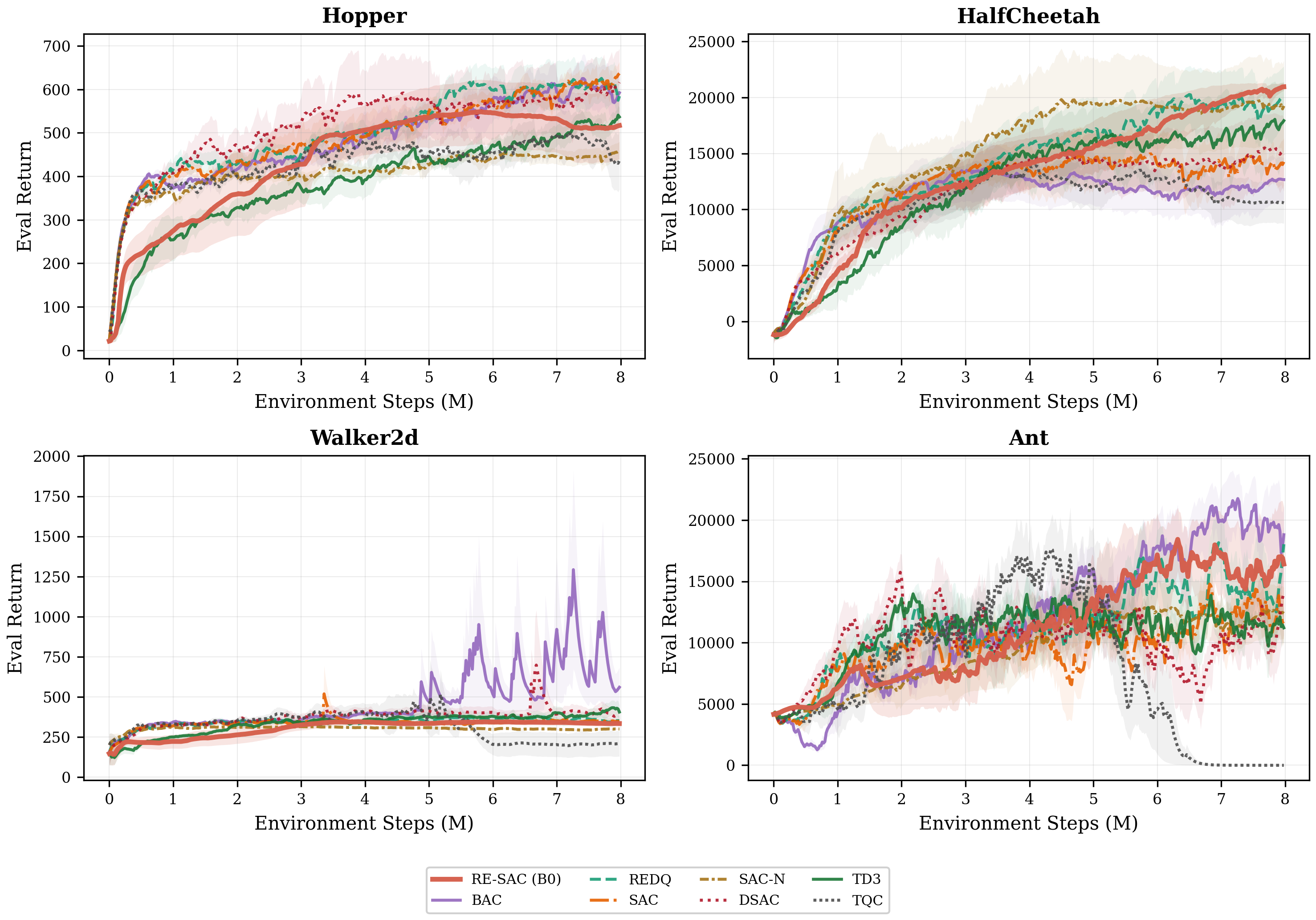}
    \caption{Five-seed evaluation returns on MuJoCo benchmarks (8M steps). Curves show the mean over seeds $\{8,16,24,32,40\}$; shaded bands show one standard error. The archived RE-SAC (B0) launch family is plotted alongside the baseline configurations; B0 uses environment-specific settings.}
    \label{fig:mujoco_curves}
\end{figure}

\subsubsection{Hyperparameter sensitivity}
\label{sec:sensitivity}
To understand which design choices in RE-SAC are load-bearing and which are forgiving, we sweep five hyperparameters one at a time on the two environments where the algorithmic behaviour differs most: Ant-v2 (high-dimensional, ensemble-friendly) and HalfCheetah-v2 (mid-dimensional, more sensitive to pessimism). Unless explicitly swept, the remaining hyperparameters follow the archived environment-specific sweep launch recipes; each configuration is trained for 8M environment steps across five seeds ($\{8,16,24,32,40\}$, post-review-round-1 update). Table~\ref{tab:sensitivity} reports mean$\pm$std of the last 25\% of evaluations.

\begin{table}[!htb]
\centering
\caption{Sensitivity sweep results (final return, mean$\pm$std over the last 25\% of evaluations and 5 seeds). Each sweep changes the listed parameter relative to its archived base launch recipe; these sweep recipes are distinct from the headline B0 family.}
\label{tab:sensitivity}
\small
\setlength{\tabcolsep}{6pt}
\renewcommand{\arraystretch}{1.06}
\begin{tabular}{@{}llcc@{}}
\toprule
Sweep & Setting & Ant-v2 & HalfCheetah-v2 \\
\midrule
$K$ (ensemble size) & 2  & diverged & diverged \\
                    & 5  & 2621$\pm$5262 & 16356$\pm$3995 \\
                    & 10 & 6829$\pm$8586 & 16171$\pm$4436 \\
                    & 20 & 7012$\pm$8920 & 16487$\pm$2803 \\
\midrule
$\lambda_{\mathrm{anc}}$ (anchor) & 0     & 5006$\pm$10036 & 15387$\pm$2312 \\
                         & 0.001 & 3871$\pm$7765  & 15534$\pm$3098 \\
                         & 0.01  & 8060$\pm$7669  & 15133$\pm$2819 \\
                         & 0.1   & 4380$\pm$8791  & 11826$\pm$3088 \\
                         & 0.5   & 2914$\pm$5856  & 16189$\pm$5150 \\
                         & 1.0   & 12092$\pm$9983 & 13524$\pm$3179 \\
\midrule
$\beta_{\mathrm{end}}$ (final LCB) & $-2.0$ & 5097$\pm$10219 & 16969$\pm$3132 \\
                          & $-1.5$ & 4783$\pm$9458  & 16171$\pm$4865 \\
                          & $-1.0$ & 2331$\pm$4694  & 17329$\pm$4109 \\
                          & $-0.5$ & 5059$\pm$10138 & 18121$\pm$4692 \\
                          & 0.0    & 11607$\pm$9509 & 15429$\pm$3947 \\
\midrule
$\alpha_{\mathrm{ind}}$ (indep.\ ratio) & 0.0  & 14045$\pm$3087 & 4199$\pm$4128 \\
                               & 0.25 & 8929$\pm$6971  & 8170$\pm$1977 \\
                               & 0.5  & 5375$\pm$6876  & 12311$\pm$3250 \\
                               & 0.75 & 4564$\pm$9154  & 13157$\pm$2978 \\
                               & 1.0  & $-$11$\pm$8    & 16609$\pm$5711 \\
\midrule
$\tau_{\mathrm{ema}}$ (EMA decay) & 0.0   & 772$\pm$1573   & 9923$\pm$5343 \\
                         & 0.001 & 4298$\pm$5841  & 6884$\pm$2917 \\
                         & 0.005 & 10009$\pm$8260 & 15413$\pm$3183 \\
                         & 0.01  & 4649$\pm$9329  & 16394$\pm$2754 \\
                         & 0.05  & 7126$\pm$9242  & 14111$\pm$2591 \\
\bottomrule
\end{tabular}
\end{table}

\paragraph{Seed sensitivity.}
The five-seed sweep exposes substantial variability on Ant: several
independent-target settings combine returns above 20,000 in some seeds
with near-zero returns in others. $K=2$ reaches NaN in the tested
configurations, and pure-independent targets
($\alpha_{\mathrm{ind}}=1$) collapse on Ant in every seed.
Positive EMA and anchoring help some settings without eliminating seed
variance. Shared targets ($\alpha_{\mathrm{ind}}=0$) are more stable on
Ant but perform poorly on HalfCheetah. The table therefore documents a
protocol-dependent stability--diversity trade-off, not universal
hyperparameter requirements or a statistically resolved ordering of LCB
coefficients.

\subsubsection{Robustness under non-stationary dynamics}
\label{sec:nonstationary}
The archived ensemble/LCB family is also evaluated under changing dynamics. To test this, we evaluate eight algorithms (RE-SAC, BAC, SAC, DSAC, TD3, REDQ, SAC-N, TQC) on a non-stationary variant of each MuJoCo environment in which the gravity scaling factor is resampled every $20{,}000$ environment steps from a log-uniform range of $[10^{-3}, 10^{3}]\times g_{\mathrm{nominal}}$, using a bank of $40$ task instances over the $8\mathrm{M}$-step training horizon. RE-SAC uses the archived environment-specific B0 recipe; baselines retain their cohort configurations. All algorithms are trained over five seeds ($\{8,16,24,32,40\}$). Table~\ref{tab:nonstationary} reports two complementary final-return metrics for each run: (i) the \emph{mean} return averaged over all $40$ evaluation tasks, and (ii) the \emph{worst-quartile} (worst-Q) return---the mean over the bottom-$25\%$ task means, a tail-performance diagnostic across regime shifts rather than a universal worst-case guarantee.

\begin{table}[!htbp]
\centering
\caption{Non-stationary final return: mean $\pm$ std (worst-quartile) across $40$ task instances and 5 seeds. Best per-environment mean and worst-Q in \textbf{bold}. Worst-Q is the mean of the bottom-$25\%$ task returns; large mean--worst-Q gap signals brittleness.}
\label{tab:nonstationary}
\small
\begin{adjustbox}{width=\textwidth}
\begin{tabular}{l|cccc}
\toprule
Algorithm & Hopper-v2 & Walker2d-v2 & HalfCheetah-v2 & Ant-v2 \\
\midrule
\textbf{RE-SAC (B0)} & $1771 \pm 701$ ($251$)        & $527 \pm 136$ ($109$) & $5328 \pm 1568$ ($740$)        & $3866 \pm 571$ ($\mathbf{1208}$) \\
BAC                & $\mathbf{2412 \pm 1337}$ ($111$) & $\mathbf{728 \pm 512}$ ($\mathbf{137}$)  & $5074 \pm 1541$ ($\mathbf{1226}$)       & $3383 \pm 672$ ($108$) \\
SAC                & $1574 \pm 923$ ($220$)        & $496 \pm 193$ ($109$)         & $4610 \pm 984$ ($-266$)      & $3912 \pm 870$ ($1093$) \\
DSAC               & $2241 \pm 1111$ ($\mathbf{265}$) & $528 \pm 160$ ($110$)         & $\mathbf{6144 \pm 1521}$ ($905$) & $\mathbf{4642 \pm 778}$ ($383$) \\
TD3                & $2061 \pm 671$ ($206$)        & $386 \pm 66$ ($87$) & $5357 \pm 1975$ ($149$)       & $3406 \pm 924$ ($274$) \\
REDQ               & $1895 \pm 810$ ($236$)        & $435 \pm 148$ ($107$) & $4825 \pm 787$ ($385$) & $3939 \pm 654$ ($1001$) \\
SAC-N               & $333 \pm 96$ ($181$) & $364 \pm 24$ ($96$) & $351 \pm 1070$ ($-952$)  & $-191 \pm 259$ ($-850$) \\
TQC                & $22 \pm 10$ ($13$) & $61 \pm 25$ ($-31$) & $-766 \pm 287$ ($-965$)      & $-2183 \pm 3097$ ($-4348$) \\
\bottomrule
\end{tabular}
\end{adjustbox}
\end{table}

\paragraph{Two-tier regime, with a protocol-specific Ant tail result.} Six algorithms (RE-SAC, BAC, SAC, DSAC, TD3, REDQ) form a competitive top tier, while the two pessimism-extreme baselines (SAC-N with $K{=}10$ critic minimum, TQC with truncated quantiles) collapse: TQC produces near-zero or negative returns across all four environments and SAC-N turns negative on the high-dimensional environments (HalfCheetah $-952$ worst-Q, Ant $-850$ worst-Q). These observations motivate comparison with the under-exploitation problem studied by~\citet{Ji2024SeizingSerendipity}; their return values alone do not identify the source of the collapse.

Among the top tier, on HalfCheetah and Ant, DSAC has the highest mean,
but its Ant worst-Q is $383$. RE-SAC has the highest reported
worst-quartile return on Ant ($1208$), ahead of SAC ($1093$), REDQ
($1001$), DSAC ($383$), TD3 ($274$), and BAC ($108$).
On Walker2d the corresponding results cluster near $100$.

\paragraph{Mean--worst-Q gap as a brittleness signal.} The gap between mean and worst-Q tells a complementary story. On Ant, RE-SAC has the smallest gap among the high-performing methods (mean $3866$, worst-Q $1208$, gap $2658$), while REDQ's gap is $2938$, SAC's is $2819$, and DSAC's is $4259$. A small gap means a method's policy generalizes more uniformly across the $40$ gravity regimes; a large gap means the method scores well \emph{on average} but has tail tasks where it falls back to near-baseline performance. In this protocol, the worst-Q column exposes a tail-performance advantage for RE-SAC on Ant, but the cross-environment results do not support a universal deployment preference.

\subsubsection{Broader dynamics-perturbation sweep}
\label{sec:jmlr_dynamics_extension}
The gravity experiment above tests one severe non-stationary regime. To check whether the conclusion is an artifact of that single perturbation family, we additionally ran a broader post-review dynamics sweep over three MuJoCo parameter families: body mass, body inertia, and joint damping. The sweep covers four algorithms (RE-SAC, BAC, REDQ, SAC), four environments, and seeds $\{8,16,24,32,40\}$, for $3\times4\times4\times5=240$ runs. For each completed run, we report the final multi-task evaluation averaged over $40$ task instances. Because raw returns are not comparable across environments, Table~\ref{tab:jmlr_dynamics_extension} summarizes the sweep using per-cell rank: within each environment--perturbation cell, algorithms are ranked by the seed-averaged final return, and the table reports mean rank together with the number of wins.

\begin{table}[!htbp]
\centering
\caption{Broader dynamics-perturbation extension. Each cell reports mean rank $\downarrow$ (wins/cells $\uparrow$) over four MuJoCo environments for the given perturbation family. Ranks are computed from final multi-task returns averaged over available finite seeds; the body-mass/HalfCheetah cell has two finite seeds for all algorithms because three seeds produced NaN final task evaluations under the most extreme body-mass settings.}
\label{tab:jmlr_dynamics_extension}
\small
\begin{adjustbox}{width=0.82\textwidth}
\begin{tabular}{lcccc}
\toprule
Algorithm & Body mass & Body inertia & Joint damping & Overall \\
\midrule
RE-SAC & 3.00 (1/4) & 2.50 (2/4) & 2.50 (1/4) & 2.67 (4/12) \\
BAC & 2.00 (1/4) & 2.00 (1/4) & \textbf{1.50 (2/4)} & \textbf{1.83 (4/12)} \\
REDQ & 2.75 (1/4) & 2.75 (1/4) & 3.75 (0/4) & 3.08 (2/12) \\
SAC & 2.25 (1/4) & 2.75 (0/4) & 2.25 (1/4) & 2.42 (2/12) \\
\bottomrule
\end{tabular}
\end{adjustbox}
\end{table}

The expanded sweep gives a more conservative picture than the gravity-only
result. BAC has the best aggregate rank, while RE-SAC wins all three Ant
perturbation cells and the HalfCheetah/body-inertia cell. The ranking
documents environment-dependent performance; no parameter-drift detection
mechanism is evaluated by this sweep.

\subsubsection{Noise-scale sensitivity of the norm proxy}
\label{sec:ipm_validation}
The sensitivity motivation in Eq.~\eqref{eq:gamma_ale} led to a separate norm-on launch cohort, whose objective and coefficients differ from the bus bonus. We therefore added a post-review injected-noise sweep on the four MuJoCo environments. Gaussian observation noise with $\sigma_{\mathrm{obs}}=0.1$ and reward noise with $\sigma_{\mathrm{rew}}=1.0$ are injected only into the training rollout, while evaluation remains clean. Table~\ref{tab:ipm_validation} reports final evaluation return over the last $25\%$ of evaluation points for RE-SAC without the IPM shift, RE-SAC with the IPM shift, and three baselines.

\begin{table}[!htbp]
\centering
\caption{Injected observation/reward noise sweep: final clean-evaluation return, mean$\pm$std over seeds $\{8,16,24,32,40\}$ unless noted. Best per environment in \textbf{bold}. The archived row labels are retained: norm-off uses zero norm/disagreement coefficients, whereas norm-on subtracts a detached target-head norm shift with $\lambda_{\mathrm{ale}}=\beta_{\mathrm{ood}}=10^{-3}$. This is a combined intervention.}
\label{tab:ipm_validation}
\small
\begin{adjustbox}{width=\textwidth}
\begin{tabular}{lcccc}
\toprule
Algorithm & Hopper-v2 & Walker2d-v2 & HalfCheetah-v2 & Ant-v2 \\
\midrule
RE-SAC (no IPM) & $281 \pm 41$ & $268 \pm 16$ & \textbf{$7{,}873 \pm 2{,}348$} & \textbf{$2{,}700 \pm 855$} \\
RE-SAC (+IPM) & $8 \pm 1$ & $30 \pm 7$ & $4{,}741 \pm 763$ & $567 \pm 1{,}418$ \\
BAC & \textbf{$370 \pm 63$} & $260 \pm 22$ & $5{,}310 \pm 1{,}580$ & $615 \pm 213$ \\
REDQ & $296 \pm 50$ & $270 \pm 33$ & $4{,}450 \pm 2{,}746$ & $2{,}509 \pm 1{,}390$ \\
SAC & $319 \pm 33$ ($n=4$) & \textbf{$300 \pm 36$} & $7{,}659 \pm 4{,}086$ & $1{,}586 \pm 1{,}045$ \\
\bottomrule
\end{tabular}
\end{adjustbox}
\end{table}

The norm-off ensemble/LCB configuration has the highest reported mean
on HalfCheetah and Ant. The norm-on/disagreement-on cohort underperforms it
on all four environments and reaches near-zero returns on Hopper/Walker2d.
This is an adverse result for the combined cohort; the table cannot
attribute it to the norm coefficient alone. The matched full/bundle bus
control in Section~\ref{sec:controlled_bus} provides the direct
positive-bonus comparison.

\paragraph{Objective-scale interpretation.}
The JAX norm-on cohort subtracts a detached target-head norm shift;
the controlled full bus recipe adds a positive continuation bonus.
These are different operations, not equivalent outer-loss sign
conventions. The noise comparison also changes replay-disagreement
regularization, so it cannot identify the norm coefficient's independent
effect. The matched bus full/bundle contrast is the direct test of the
positive bonus. Together, these results identify adverse settings and
objective-scale sensitivity rather than justify a universal norm schedule.

\subsection{Bus fleet control application}
\label{sec:bus_experiments}
The archived application cohort complements the matched mechanism controls in Section~\ref{sec:controlled_bus}. It uses the original bus recipes and its own seeds and evaluation windows.

\subsubsection{Simulation environment}
We evaluate all algorithms on the bidirectional bus corridor simulation developed in our previous work~\citep{Zhang2026SingleAgentBus}. The environment models a realistic timetabled bus line with 22 stops (including two terminals) operating over a 13-hour service window (06:00--19:00). Buses are dispatched every 360 seconds from both terminals with a 180-second directional offset. Upon completing a trip, each vehicle becomes idle at its terminal and is available for reassignment according to the timetable, thereby modeling dynamic fleet sizing.

\textbf{Passenger demand} is specified by an hourly origin-destination (OD) matrix. For each pair of stations and each time period, passenger arrivals are sampled from a Poisson distribution with the corresponding time-varying rate. \textbf{Traffic conditions} are modeled as stochastic inter-stop travel speeds sampled from Gaussian distributions with hourly varying means and a fixed standard deviation of $\sigma_{\mathrm{speed}} = 1.5$~m/s, capturing realistic congestion fluctuations.

\textbf{MDP formulation.} The holding control problem is formulated as an event-driven MDP~\citep{Zhang2026SingleAgentBus}. Decisions are triggered after all passenger boarding/alighting activities have completed at each stop. The \textit{state} $s_t$ consists of categorical features (Bus~ID, Station~ID, Direction, Time~Period) and continuous features (forward headway $h_f$, backward headway $h_b$, current segment speed $v$). The \textit{action} $a_t \in [0, H_{\max}]$ specifies the additional holding time at the current stop. The \textit{reward} follows the ridge-shaped function of~\citet{Zhang2026SingleAgentBus}, which simultaneously incentivizes headway regularity ($h_f \approx h_b \approx h_{\mathrm{sched}}$) and penalizes extreme deviations from the scheduled headway $h_{\mathrm{sched}}$.

\subsubsection{Network architecture and hyperparameters}
All algorithms are implemented in PyTorch. The $K$ critic networks are implemented as a single vectorized module (VectorizedCritic) for computational efficiency, each consisting of a 4-layer MLP with hidden dimensions $[64, 64, 64]$ and rectified linear unit (ReLU) activations. The actor network is a separate 4-layer MLP with the same hidden dimension, outputting mean and log-standard-deviation for a squashed Gaussian policy. Following~\citet{Zhang2026SingleAgentBus}, categorical features (Bus~ID, Station~ID, Direction, Time~Period) are encoded via learnable embedding layers with dimension $d_{\mathrm{emb}} = \min(50, \lfloor C_i / 2 \rfloor)$, where $C_i$ is the number of unique categories. The embedded representations are concatenated with continuous state features (forward headway, backward headway, segment speed) to form the input vector.

The key hyperparameters used in the bus-control experiments are:
\begin{center}
\small
\setlength{\tabcolsep}{5pt}
\renewcommand{\arraystretch}{1.12}
\begin{tabular}{@{}p{0.20\linewidth}p{0.30\linewidth}p{0.42\linewidth}@{}}
\toprule
Group & Hyperparameter & Value \\
\midrule
Model & Ensemble size & $K \in \{2,5,10,20,40\}$; default $K=10$ unless otherwise specified \\
Model & Embedding dimension & $d_{\mathrm{emb}}=\min(50,\lfloor C_i/2 \rfloor)$ for each categorical feature \\
\midrule
Optimization & Learning rate & $1{\times}10^{-5}$ for critic, actor, and temperature (Adam) \\
Optimization & Batch / buffer & Batch size $2048$; replay buffer $10^6$ transitions \\
Optimization & Discount / target update & $\gamma=0.99$; target smoothing coefficient $\tau_{\mathrm{tgt}}=0.01$ \\
Optimization & Update schedule & Gradient clipping max norm $1.0$; critic:actor update ratio $2{:}1$ \\
\midrule
Robustness & Norm smoothness proxy & $\lambda_{\mathrm{ale}}=0.01$; frozen target-network $\ell_1$ shift in the Bellman target \\
Robustness & Shifted actor LCB & $\beta_{\mathrm{lcb}}=-2$ in $\bar z+\beta_{\mathrm{lcb}}\sigma_z$ \\
Robustness & Replay disagreement & $\beta_{\mathrm{ood}}=0.01$; penalizes ensemble standard deviation in critic loss \\
\midrule
Entropy & Temperature $\alpha_{\mathrm{ent}}$ & Automatically tuned with target entropy $-\dim(\mathcal{A})$ and cap $\alpha_{\mathrm{ent,max}}=0.6$ \\
\bottomrule
\end{tabular}
\end{center}
All baseline algorithms (SAC, DSAC, BAC) share the same network architecture, learning rate, and buffer configuration for fair comparison. Each experiment is trained for 500 episodes with training updates triggered every 5 environment steps, and evaluation metrics are averaged over 10 independent rollouts.

\subsubsection{Baselines and ablations}
The original archive contains the SAC-labelled two-head control, DSAC
and BAC references, RE-SAC ensemble sizes $K\in\{2,5,10,20,40\}$, and
two historical ablation labels. The SAC-labelled bus control uses
headwise targets and a mean actor with norm, BC, and disagreement
coefficients zero; it is not the clipped-target SAC control used in the
mechanism study. The label \emph{Epistemic-Only} denotes $K=10$,
zero norm and BC coefficients, LCB coefficient $-2$, and critic
disagreement coefficient 0.01. \emph{Aleatoric-Only} retains $K=10$
headwise critics and the positive norm bonus 0.01, but removes LCB,
BC, and disagreement terms. These labels are preserved in the archived
plots, not interpreted as identified uncertainty channels.
DSAC and BAC retain their seed-8 diagnostic status.

\paragraph{Classical transit baselines and operational metrics.}
Because cumulative reward is only a proxy for service quality, we also report two non-learning control baselines on the same simulator seed: \textbf{No holding}, which dispatches buses according to the timetable and applies zero extra holding at every control event, and \textbf{headway-equalising holding}, a classical feedback rule that applies
\[
a_t=\mathrm{clip}\left(\frac{h_b-h_f}{2},\,0,\,60\right),
\]
thereby holding a bus only when it is closer to the preceding bus than to the following bus. These controllers do not train and therefore are not plotted as learning curves in Fig.~\ref{fig:all_experiments}. Instead, Table~\ref{tab:classical_bus_baselines} reports one seeded rollout (seed~8) using passenger-facing operational metrics computed directly from the simulator: average passenger waiting time, average in-vehicle travel time, mean absolute headway imbalance $|h_f-h_b|$, and bunching rate, defined as the fraction of control events with $\min(h_f,h_b)<180$ seconds. The comparison shows that the classical headway rule is a strong operational sanity check: it sharply reduces bunching relative to no holding, but it also increases in-vehicle travel time through added holding. The learned-policy results below should therefore be read as evidence for robust DRL training stability and value-estimation quality, not as a complete claim of dominance over every possible handcrafted transit controller.

\begin{table}[!htbp]
\centering
\caption{Classical bus-holding sanity baselines on the bidirectional corridor simulator (one seeded rollout, seed~8). Reward is cumulative environment reward; operational metrics are computed from passenger and headway traces.}
\label{tab:classical_bus_baselines}
\small
\begin{tabular}{@{}lccccc@{}}
\toprule
Controller & Reward ($10^6$) & Wait (min) & Travel (min) & $|h_f-h_b|$ (s) & Bunching rate \\
\midrule
No holding & $-0.841$ & 3.76 & 10.93 & 127.1 & 13.4\% \\
Headway-equalising & $-0.437$ & 3.62 & 12.19 & 59.2 & 0.7\% \\
\bottomrule
\end{tabular}
\end{table}

\subsubsection{Overall performance and learning stability}
Fig.~\ref{fig:all_experiments} displays the smoothed cumulative rewards over the available bus training-seed archive. For SAC, the RE-SAC ensemble-size variants, and the two channel ablations, the curves aggregate seed~8 with the two new post-review seeds (32 and 40); DSAC-v1 and BAC are retained as seed-8 reference curves because they were not part of the post-review bus seed-extension matrix. The SAC-labelled archive control and the Aleatoric-Only ablation remain below the moderate RE-SAC ensembles, while the Epistemic-Only variant shows the clearest instability: its available-seed last-100 reward is $-0.93 \times 10^6$ and one seed collapses far below the rest of the archive. This is consistent with the hypothesis that this unshifted disagreement-control bundle is poorly matched to this highly stochastic recipe; the result is an observed failure mode, not proof that the variance signal statistically identifies environmental noise.

The post-review seed extension reran the bus-control variants for two additional independent training seeds (32 and 40) using the same 500-episode budget and 10 deterministic evaluation rollouts. Combined with the archived seed-8 learning curves, the available training-seed evidence keeps the same qualitative ordering: the moderate ensembles remain the most reliable, while the epistemic-only ablation remains the least stable. Table~\ref{tab:bus_seed_extension} reports the full variant-level seed-extension summary. Specifically, the available seed-8/32/40 last-100-episode training rewards are $-435\text{k}\pm16\text{k}$ for $K{=}10$, $-434\text{k}\pm42\text{k}$ for $K{=}20$, $-504\text{k}\pm30\text{k}$ for SAC, and $-926\text{k}\pm169\text{k}$ for Epistemic-Only. The deterministic operational evaluations on the two new seeds are consistent with this picture: $K{=}10$ obtains $-446\text{k}\pm2\text{k}$ and $K{=}20$ obtains $-441\text{k}\pm26\text{k}$, compared with $-482\text{k}\pm28\text{k}$ for SAC and $-863\text{k}\pm130\text{k}$ for Epistemic-Only. These are training-seed differences, not merely evaluation-rollout variation.

\begin{table}[!htbp]
\centering
\caption{Post-review bus seed-extension summary. Rewards are negative cumulative rewards, so less negative is better. The training column combines the archived seed~8 learning curve with the two new training seeds (32 and 40), using the last 100 training episodes. The operational column reports deterministic evaluation rewards on the two new seeds only.}
\label{tab:bus_seed_extension}
\scriptsize
\begin{adjustbox}{width=\textwidth}
\begin{tabular}{lcc}
\toprule
Variant & Last-100 training reward, seeds 8/32/40 $\uparrow$ & Operational eval reward, seeds 32/40 $\uparrow$ \\
\midrule
SAC & $-504\text{k}\pm30\text{k}$ & $-482\text{k}\pm28\text{k}$ \\
RE-SAC ($K{=}2$) & $-592\text{k}\pm176\text{k}$ & $-443\text{k}\pm10\text{k}$ \\
RE-SAC ($K{=}5$) & $-569\text{k}\pm217\text{k}$ & $-557\text{k}\pm117\text{k}$ \\
RE-SAC ($K{=}10$) & $-435\text{k}\pm16\text{k}$ & $-446\text{k}\pm2\text{k}$ \\
RE-SAC ($K{=}20$) & $\mathbf{-434\text{k}\pm42\text{k}}$ & $\mathbf{-441\text{k}\pm26\text{k}}$ \\
RE-SAC ($K{=}40$) & $-458\text{k}\pm15\text{k}$ & $-496\text{k}\pm15\text{k}$ \\
Aleatoric-Only & $-519\text{k}\pm19\text{k}$ & $-554\text{k}\pm34\text{k}$ \\
Epistemic-Only & $-926\text{k}\pm170\text{k}$ & $-863\text{k}\pm130\text{k}$ \\
\bottomrule
\end{tabular}
\end{adjustbox}
\end{table}

In the available training-seed archive, moderate ensemble sizes ($K=10$ or
$20$) have the highest mean final reward (around $-0.43\times10^6$).
The smaller ensembles have greater training-seed variation, while $K=40$
has lower mean reward. These observations complement the broader
ensemble-pessimism trade-off studied by~\citet{An2021UncertaintyEnsemble},
in an online bus-control rather than offline-learning setting.
The BAC baseline, designed to avoid underestimation by exploiting past successes, performs relatively well but exhibits higher variance during training compared with the best archived RE-SAC variants. The matched target-clipping control provides more specific evidence: replacing the min target by a mean target improves the noisy-bus return and reduces the native task-Q bias in all three controlled seeds. The result supports a clipping-associated underestimation mechanism; it does not establish a statistical aleatoric/epistemic decomposition.

\begin{figure}[!htbp]
    \centering
    \includegraphics[width=\linewidth]{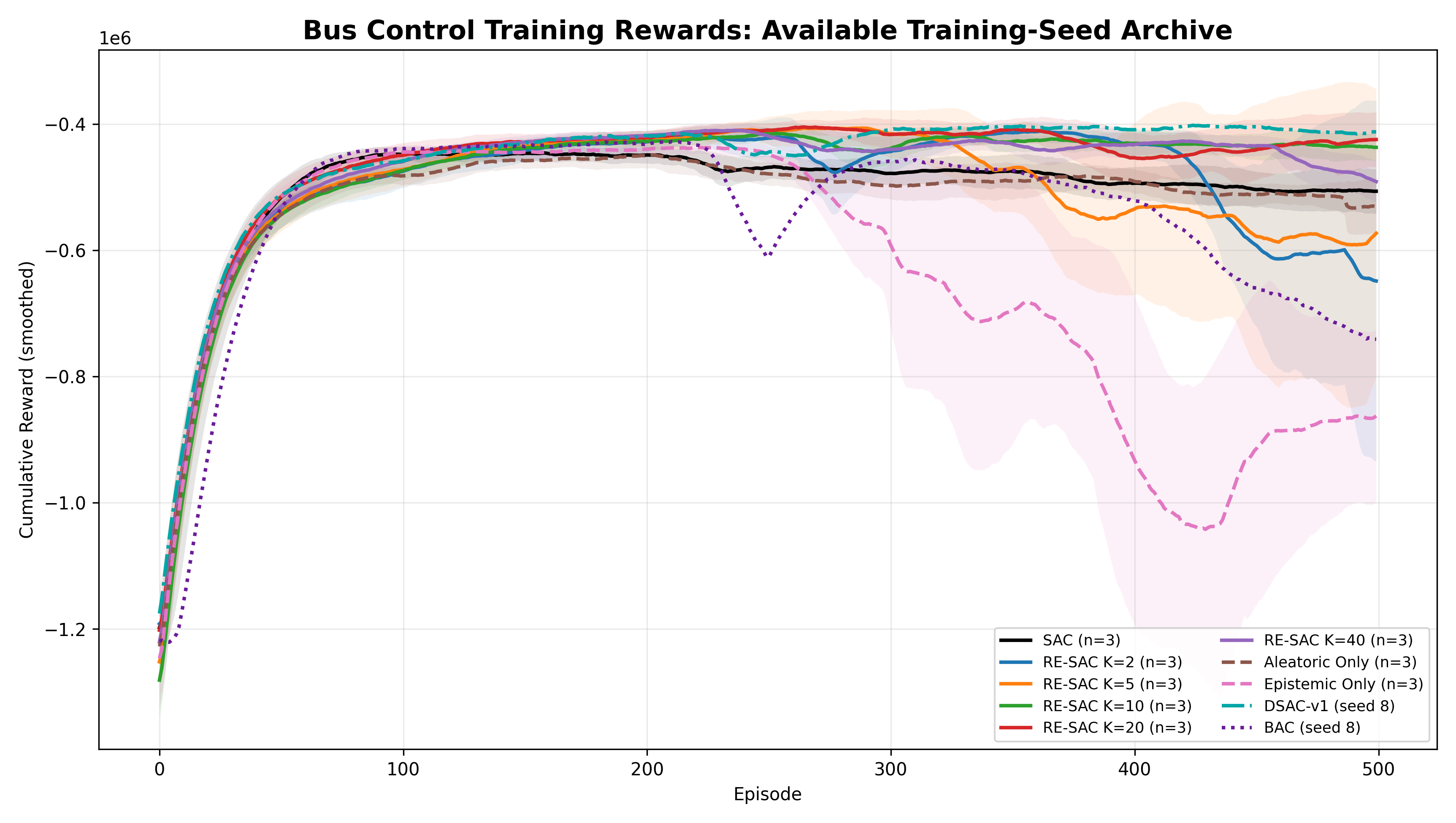}
    \caption{Learning curves of cumulative reward for the bus-control experiments. SAC, RE-SAC variants, and the channel ablations aggregate the available training seeds (8/32/40); shaded bands show one standard deviation across seeds. DSAC-v1 and BAC are shown as seed-8 reference curves.}
    \label{fig:all_experiments}
\end{figure}

\subsubsection{Q-value estimation stability}
The catastrophic drop in the Epistemic-Only policy is accompanied by extreme negative Q-values. Its available-seed mean Q-value curve falls below $-10^6$ on the symlog scale, whereas the moderate RE-SAC ensembles remain in a more stable range and the SAC-labelled control drifts mildly negative. The archived trajectories are reproduced in the online appendix. The completed full/bundle control further shows that the norm-augmented arm can create a large native task-Q offset while matching its augmented reference.

\subsubsection{Robustness in rare states (Mahalanobis rareness)}
While the preceding Q-value analysis tracks estimation stability \emph{over training episodes}, a robust optimization perspective motivates a complementary view: \emph{at the same level of state rarity, does the algorithm produce smaller estimation errors than baselines?} This distinction is useful because rare-state diagnostics can reveal behavior hidden by average-case returns, although the aligned Oracle metric below does not establish bounded task-Q error.

We fit a two-dimensional Gaussian to the forward/backward headway
vector $x=(h_f,h_b)$ within each Station ID $\times$ Direction group.
The group mean $\mu_{\mathrm{Maha}}$ and covariance $\Sigma_{\mathrm{Maha}}$
define rareness as
$\sqrt{(x-\mu_{\mathrm{Maha}})^\top
\Sigma_{\mathrm{Maha}}^{-1}(x-\mu_{\mathrm{Maha}})}$.
Grouping separates terminal and mid-route regimes; covariance scaling
accounts for headway correlation. Rareness measures distance from the
fitted distribution, not epistemic uncertainty itself.

Fig.~\ref{fig:mahalanobis} bins the MAE between each algorithm's Q-value estimates and the MC trajectory return $Q_{\mathrm{MC}}(s,a)$ by Mahalanobis rareness. Here, $Q_{\mathrm{MC}}$ is computed offline from the recorded evaluation trajectories as the discounted cumulative reward from each decision point onward: $Q_{\mathrm{MC}}(s_t, a_t) = \sum_{k=0}^{T_{\mathrm{eval}}-t} \gamma^k r_{t+k}$, providing a realized trajectory-return reference independent of the learned critic; a single return is not the conditional expected value.

Since different algorithms learn Q-values at different absolute scales, direct comparison of raw prediction errors would be misleading. We therefore apply a per-algorithm z-score alignment that rescales each algorithm's predictions to match the mean and standard deviation of $Q_{\mathrm{MC}}$ before computing the error. Additionally, for ensemble-based algorithms, we report the \textbf{Oracle MAE}: for each state, we select the ensemble head whose aligned prediction is closest to $Q_{\mathrm{MC}}$, and compute the MAE over these best-head predictions. This oracle serves as a favorable upper bound on the ensemble's representational capacity---it measures how well the ensemble \textit{could} estimate the true return if paired with an ideal head-selection mechanism.

As states become increasingly rare (moving left to right on the x-axis), the estimation errors for Vanilla SAC and DSAC-v1 escalate significantly, while RE-SAC (Ensemble) has the lowest aligned Oracle MAE across the reported rareness bins (overall average of 1647, compared with SAC's 4343 and DSAC's 5945). This Oracle/z-score diagnostic measures favorable representational capacity under the stated alignment; it is not a calibrated task-Q guarantee or a proof that the two uncertainty sources have been separated.
\vspace{-2pt}
\begin{figure}[!htbp]
    \centering
    \includegraphics[width=\linewidth]{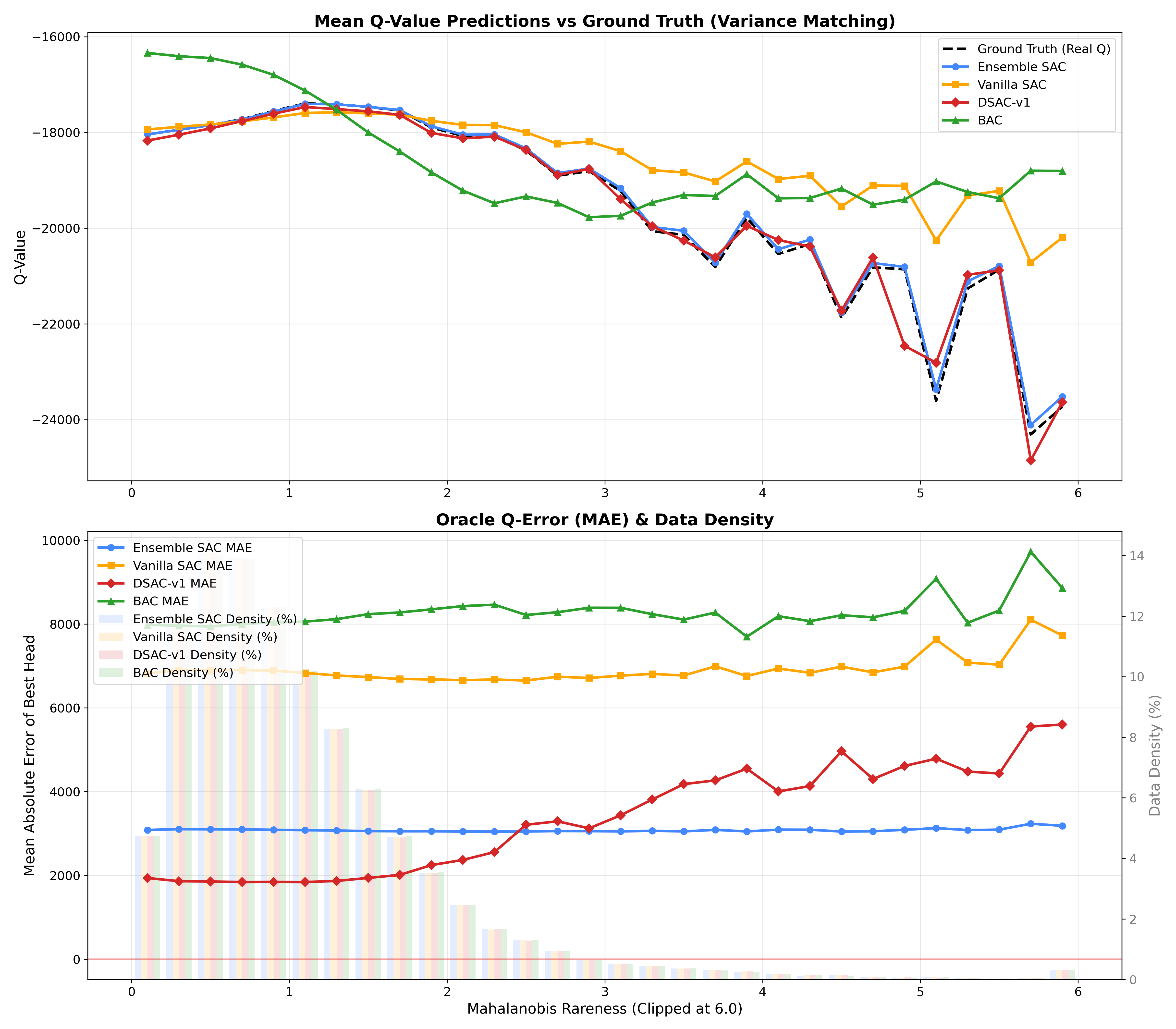}
    \vspace{-6pt}
    \caption{Aligned Oracle Q-error by Mahalanobis rareness. Per-algorithm z-score alignment and ideal head selection make this a representation diagnostic, not native task-Q calibration.}
    \label{fig:mahalanobis}
\end{figure}

\FloatBarrier
\section{Limitations}
The controlled target and gradient results use one noisy-bus protocol and
three reused training seeds. Online policies and replay distributions
diverge after intervention, so the matched-input algebra does not establish
training-time mediation. Mean aggregation can overestimate values; the
observed reduction in absolute task-Q bias does not imply universal
calibration. The original full method's norm contribution is unsupported
by the matched return control, and its bonus changes the evaluation
objective. The continuous-control B0 family, archived bus recipe, and
current controls are different implementations. Oracle errors use ideal
head selection and per-algorithm alignment, not a deployable estimator.
The formal results cover the stated finite operators or fixed-policy
abstractions; they do not identify statistical uncertainty components or
prove neural actor--critic convergence.

\section{Conclusion}
RE-SAC provides a case study of target aggregation and value-objective
shifts in noisy off-policy control. The matched $K=10$ bus intervention
improves return and native task-Q calibration when shared-min targets
are replaced by shared-mean targets. A separate actor-gradient control
improves return, while the corrected full norm-augmented arm is worse
than its no-norm bundle in every seed. The finite-ensemble and
continuation-bonus identities explain why reducing clipping, raising
augmented value, and improving policy performance are distinct claims.
The existing continuous-control sweeps and archived bus comparisons
define the tested performance boundaries. The supported result is a
clipping-associated underestimation effect in the noisy-bus protocol,
not a demonstrated dual-uncertainty decomposition.

\acks{This work was supported by the National Natural Science Foundation of China (Grant 72371251), the Natural Science Foundation for Distinguished Young Scholars of Hunan Province (Grant 2024JJ2080), and the Key Research and Development Program of Hunan Province of China (Grant 2024JK2007). Code is available at \url{https://github.com/erzhu419/RE-SAC}. The authors declare no competing interests.}

\bibliography{ensemble_refs_jmlr}
\end{document}